\documentclass[final,3p,times,review,preprint,onecolumn]{elsarticle}

\usepackage[utf8]{inputenc} 
\usepackage{textgreek}
\usepackage[T1]{fontenc}    
\usepackage{url}            
\usepackage{booktabs}       
\usepackage{amsfonts}       
\usepackage{nicefrac}       
\usepackage{microtype}      
\usepackage[table]{xcolor}
\usepackage{lineno}
\usepackage[labelfont=bf]{caption}
\usepackage{colortbl}
\usepackage{graphicx}

\usepackage{soul,framed} 

\colorlet{shadecolor}{yellow}

\UseRawInputEncoding
\usepackage{array}
\usepackage{mdwmath}
\usepackage{mdwtab}
\usepackage{eqparbox}
\usepackage{url}
\usepackage{multirow}
\usepackage{graphicx}
\usepackage{amsfonts}
\usepackage{amsmath}
\usepackage{amsthm}
\usepackage{amssymb}
\usepackage{multicol}
\usepackage{multirow}
\usepackage{booktabs}
\usepackage{interval}
\usepackage{adjustbox}
\usepackage[colorlinks,urlcolor=purple,linkcolor=purple,citecolor=purple]{hyperref}
\usepackage{float}
\usepackage{fancyhdr}
\usepackage{diagbox}

\newtheorem{proposition}{Proposition}

\usepackage{cases} 
\newcommand{\indicator}[1]{\mathbb{I}[#1]}
\newcommand{\pthat}[1]{\hat{p}_T^{(#1)}}
\newcommand{\pttrue}[1]{p_T(#1)}

\definecolor{lightblue}{rgb}{0.9, 0.95, 1.0}
\definecolor{mygreen}{rgb}{0.01, 0.5, 0.01}
\definecolor{myred}{rgb}{0.8, 0.01, 0.01}
\usepackage[caption=false, font=footnotesize]{subfig}

\begin{document}

\begin{frontmatter}



\title{Physics-Informed Graph Neural Networks for Transverse Momentum Estimation in CMS Trigger Systems}

\author[1]{Md Abrar Jahin\corref{corauthor}}
\ead{jahin@usc.edu, abrar.jahin.2652@gmail.com}
\author[2]{Shahriar Soudeep}
\ead{20-43823-2@student.aiub.edu}
\author[2]{M. F. Mridha\corref{corauthor}}
\ead{firoz.mridha@aiub.edu}
\author[3]{Muhammad Mostafa Monowar}
\ead{mmonowar@kau.edu.sa}
\author[3]{Md. Abdul Hamid}
\ead{mabdulhamid1@kau.edu.sa}

\affiliation[1]{organization={Thomas Lord Department of Computer Science},
    addressline={University of Southern California}, 
    city={Los Angeles},
    state={CA},
    postcode={90089}, 
    country={USA}}
\affiliation[2]{organization={Department of Computer Science},
    addressline={American International University-Bangladesh}, 
    city={Dhaka},
    postcode={1229},
    country={Bangladesh}}
\affiliation[3]{organization={Department of Information Technology
},
    addressline={King AbdulAziz University}, 
    city={Jeddah},
    postcode={21589},
    country={Saudi Arabia}}

\cortext[corauthor]{Corresponding author}

\begin{abstract}
Real-time particle transverse momentum ($p_T$) estimation in high-energy physics demands algorithms that are both efficient and accurate under strict hardware constraints. Static machine learning models degrade under high pileup and lack physics-aware optimization, while generic graph neural networks (GNNs) often neglect domain structure critical for robust $p_T$ regression. We propose a physics-informed GNN framework that systematically encodes detector geometry and physical observables through four distinct graph construction strategies that systematically encode detector geometry and physical observables: station-as-node, feature-as-node, bending angle-centric, and pseudorapidity ($\eta$)-centric representations. This framework integrates these tailored graph structures with a novel Message Passing Layer (MPL), featuring intra-message attention and gated updates, and domain-specific loss functions incorporating $p_{T}$-distribution priors. Our co-design methodology yields superior accuracy-efficiency trade-offs compared to existing baselines. Extensive experiments on the CMS Trigger Dataset validate the approach: a station-informed EdgeConv model achieves a state-of-the-art MAE of 0.8525 with $\ge55\%$ fewer parameters than deep learning baselines, especially TabNet, while an $\eta$-centric MPL configuration also demonstrates improved accuracy with comparable efficiency. These results establish the promise of physics-guided GNNs for deployment in resource-constrained trigger systems. Our data and code are available at: \url{https://github.com/Abrar2652/gnn-particle-momentum-cms-trigger}.
\end{abstract}



\end{frontmatter}


\section{Introduction}
Accurate real-time estimation of transverse momentum (\(p_T\)) is crucial for trigger-level decisions in high-energy physics (HEP) experiments, such as the Compact Muon Solenoid (CMS) at CERN. While traditional rule-based methods meet sub-millisecond latency requirements, their accuracy degrades under high pileup, where overlapping particle signatures confuse deterministic models \cite{collaboration_cms_2017, hayrapetyan_performance_2024,silva_characterizing_2015,andreeva_high-energy_2008,newman_high_2006}. Static systems like Boosted Decision Trees (BDTs) require full retraining for updates and lack adaptability to changing detector conditions \cite{acosta_boosted_2018,sekhar_stochastic_2025,khataei_treelut_2025,mahmood_enhancing_2024,antoniuk_ensemble_2024,griffiths_gradient_2024}.

To address these challenges, the CMS experiment has introduced machine learning algorithms at the Level-1 (hardware) trigger for real-time momentum estimation. The initial implementation used a discretized boosted decision trees~\cite{acosta_boosted_2018}, while current efforts explore deep learning (DL) models capable of microsecond-level inference \cite{cms_trigger_ml_2024}. Among these, Graph Neural Networks (GNNs) show strong potential for momentum regression by modeling spatial and temporal correlations in detector hits as muons traverse sequential endcap stations. Representing this hit data as graphs enables GNNs to learn fine-grained patterns, improving momentum resolution and enhancing the trigger's ability to discriminate high-$p_T$ signal muons from low-$p_T$ background \cite{cms_gnn_momentum_2024}. GNNs have shown strong performance in HEP tasks like track reconstruction and jet tagging by modeling relational data \cite{duarte_fast_2018, shlomi_graph_2020,sun_fast_2025,stein_novel_2024,lange_tau_2024}. However, these architectures are often parameter-heavy and designed for offline analysis. They fail to meet the real-time constraints of trigger systems, particularly when deployed on resource-limited hardware such as field-programmable gate arrays (FPGAs) \cite{iiyama_distance-weighted_2021,brignone_making_2025,bosio_low-power_2025,hozhabr_survey_2025,lahti_high-level_2025}.

We address this gap by introducing physics-informed GNN architectures tailored for low-latency, high-accuracy $p_T$ regression. Our framework systematically explores how encoding detector geometry and physical observables, such as bending angles and pseudorapidity ($\eta$), into four distinct graph construction strategies can enhance predictive performance. Complementing these structural innovations, we develop domain-specific loss functions that incorporate knowledge of $p_T$ distributions to improve model generalization across varying momentum regimes. Furthermore, we introduce a novel Message Passing Layer (MPL) architecture designed for efficient and expressive feature learning from these physics-aware graphs. These innovations enable our models to achieve superior accuracy and efficiency under the stringent deployment constraints of trigger systems.

\textbf{Key Contributions:}
\textbf{(i)} We propose and evaluate four distinct graph construction strategies that systematically encode detector geometry and critical physical observables (like $\eta$ and station-based features) to enhance $p_T$ regression; \textbf{(ii)} We introduce a new MPL featuring intra-message attention and gated updates, tailored for robust and efficient feature extraction from the structured particle data; \textbf{(iii)} We develop custom loss functions, including a $p_T$-weighted regression loss and an asymmetric penalty loss, which leverage domain knowledge of momentum distributions to improve model generalization and stability; \textbf{(iv)} Our best physics-informed GNNs significantly outperform established baseline; \textbf{(v)} We validate our framework through extensive experiments on the CMS Trigger Dataset, showing its strong potential for deployment in real-time, resource-constrained trigger systems.

\section{Related Work}
\label{related_work}

\subsection{Traditional Momentum Estimation in Particle Physics}
Momentum estimation in the CMS Trigger System has historically relied on rule-based algorithms that compute $p_T$ using geometric approximations of particle trajectories in magnetic fields \cite{jones_development_2007, olsen_method_1966}. While these methods meet strict latency requirements ($\leq$1 ms/event), their accuracy degrades significantly in high-pileup environments due to overlapping particle signatures \cite{das_overview_2022}. Recent efforts to prepare for the High-Luminosity LHC focus on hardware parallelization \cite{apollinari_high_2015}, but static frameworks like the CMS Endcap Muon Track Finder's 230 Boosted Decision Trees (BDTs) \cite{acosta_boosted_2018} suffer from inflexibility, requiring full retraining for updates and lacking adaptability to dynamic detector conditions. This rigidity highlights a critical gap in deployable, learnable frameworks for real-time momentum estimation.

\subsection{GNNs in HEP}
Graph-based machine learning has advanced tasks like track reconstruction \cite{liu_hierarchical_2023,re_construct} and jet tagging \cite{ma_jet_2023} by modeling detector hits as nodes and spatial relationships as edges. However, existing GNNs in particle physics primarily target classification or localized reconstruction \cite{shlomi_graph_2020}, neglecting direct momentum regression. Offline evaluations further limit their relevance to latency-sensitive systems like the CMS Trigger \cite{collaboration_cms_2017}, while parameter-heavy architectures hinder deployment on resource-constrained FPGAs. For instance, frameworks like \texttt{hls4ml} \cite{aarrestad_fast_2021,hawks_wa-hls4ml_2025,guglielmo_end--end_2025,lu_automatic_2024,que_ll-gnn_2024,kim_omega_2025,zhang_graphagile_2023} enable low-latency GNN inference but omit momentum estimation from their scope, leaving a critical gap in real-time applications.

\subsection{Physics-Informed GNNs for Real-Time Momentum Estimation}
Prior GNNs in HEP simplify graph topologies or treat detector hits as nodes, limiting their expressiveness for regression tasks. We bridge this gap by redefining nodes as physical quantities (e.g., $\eta$, bending angles) and edges as angular relationships (e.g., $\sin(\phi)$, $\cos(\phi)$), aligning graph structure with domain priors. Unlike classification-focused GNNs \cite{liu_survey_2023,chen_are_2020,zhu_focusedcleaner_2023,xie_semisupervised_2022}, our method introduces task-specific losses (e.g., $p_T$-weighted regression) and attempts to outperform the DL baselines, including the state-of-the-art TabNet, even under parameter-constraint settings. By co-designing parameter-efficient architectures and physics-aware edge features, we advance beyond static BDTs \cite{acosta_boosted_2018} and generic GNNs, offering the first framework for deployable, real-time momentum estimation in the CMS Trigger System.

\section{Methodology}
\label{methodology}
Our methodology for real-time \(p_T\) estimation centers on physics-informed GNNs, systematically co-designing data representation, model architecture, and learning objectives to optimize for both accuracy and efficiency. The formal problem statement is detailed in Section~\ref{sec:statement}. This section details our core contributions: (i) four distinct physics-informed graph construction strategies that encode detector geometry and particle kinematics into varied graph structures (detailed in Section~\ref{graph_construction}); (ii) the architecture of a novel MPL, featuring intra-message attention and gated updates, tailored for robust and efficient feature extraction from these structured particle data (Section~\ref{model_architecture}); and (iii) the development of domain-specific loss functions which leverage knowledge of \(p_T\) distributions to improve model generalization and stability across critical momentum regimes (Section~\ref{loss_function}). These components are subsequently integrated into end-to-end GNN models for $p_T$ prediction.

\subsection{Formal Problem Statement}
\label{sec:statement}
Let $\mathcal{X}_{raw}$ be the space of raw detector measurements for a particle event.
Our primary input is a set of $N_s$ station measurements $S = \{s_1, s_2, \ldots, s_{N_s}\}$, where each $s_i \in \mathbb{R}^{d_s}$ is a feature vector from a detector station.
The goal is to learn a function $f_\theta: \mathcal{P}(S) \rightarrow \mathbb{R}^+$ (where $\mathcal{P}(S)$ is a representation derived from $S$, typically a graph) parameterized by $\theta$,
that predicts the transverse momentum $p_T^{true} \in \mathbb{R}^+$.
The objective is to minimize the expected loss $\mathcal{L}$:
\begin{equation}
    \min_\theta \mathbb{E}_{(S, p_T^{true}) \sim \mathcal{D}} [\mathcal{L}(f_\theta(\mathcal{G}(S)), p_T^{true})]
\end{equation}
where $\mathcal{D}$ is the data distribution, and $\mathcal{G}(S)$ is a graph constructed from the station measurements $S$. The function $f_\theta$ is realized by a GNN.

\subsection{Dataset}
We use the \textit{CMS Trigger Dataset}\footnote{\url{https://www.kaggle.com/datasets/ekurtoglu/cms-dataset}}, consisting of 1,179,356 samples generated from simulated muon events using Pythia 8~\cite{pythia8}. Each sample corresponds to a candidate muon track passing through the CMS endcap detectors and consists of 31 features. As muons traverse the detector, they may leave hits in up to four sequential stations, labeled 1 through 4. The specific combination of stations with hits defines the track "mode". From each hit detector station, 7 features are extracted: $\phi$ and $\theta$ coordinates, Bending angle, Timing information, Ring number, Front/Rear hit indicator, and a Mask flag. In addition, three global "road" variables--Pattern Straightness, Zone, and Median $\theta$--are included, yielding a total of $7 \times 4 + 3 = 31$ features per event.

\subsection{Preprocessing}
\label{sec:preprocessing}
We begin by preprocessing the dataset, which records detector responses from the CMS muon trigger system across multiple stations. Each sample contains localized measurements of angular coordinates, time, bending angles, and the charge-weighted inverse transverse momentum ($q/p_T$), which we transform to obtain the target variable. Specifically, the transverse momentum $p_T$ for each sample is computed by taking the absolute reciprocal of the $q/p_T$ value:
\begin{equation}
p_T = \left| \frac{1}{(q/p_T)_{\text{sample}}} \right|
\label{eq:pt}
\end{equation}
This transformation standardizes the regression target and ensures numerical stability, particularly in low-momentum regions.

We compute trigonometric and $\eta$-derived features from the provided angular measurements to enrich the feature space with kinematically informative attributes. Let $\Phi_i$ and $\Theta_i$ denote the azimuthal and polar angles for detector station $i \in {0,1,2,3}$. For each station, the sine and cosine of the azimuthal angles, $\sin\Phi_i$ and $\cos\Phi_i$, are computed to preserve rotational information while mitigating discontinuities inherent in angular variables. In addition, for stations $i \in {0,2,3,4}$, we derive $\eta$ values using the standard transformation:
\begin{equation}
\eta_i = -\ln (\tan \left( \frac{\Theta_i}{2} \right))
\label{eq:eta}
\end{equation}
which maps polar angles to a scale-invariant coordinate system commonly used in HEP.

We apply interquartile-range (IQR) filtering to $p_T$ values to reduce sensitivity to extreme outliers in the target distribution when constructing the datasets used in Methods 3 and 4. Let $Q_1$ and $Q_3$ represent the first and third quartiles of the $p_T$ distribution, and define $\text{IQR} = Q_3 - Q_1$. All samples with $p_T$ values outside the interval of \(\left[ Q_1 - 1.5 \times \text{IQR},~Q_3 + 1.5 \times \text{IQR} \right]\label{eq:iqr}\) are removed. After this filtering step, 1,029,592 samples remain for model training and evaluation.

Feature selection differs across graph construction methodologies. Methods 1 and 2 employ the full set of 28 features extracted from the raw data, following the computation of $p_T$. These include angular and timing information from each station and additional categorical indicators such as RingNumber, Front, and Mask. In contrast, Methods 3 and 4 utilize a reduced 16-dimensional feature space comprising engineered variables, including $\sin\Phi_i$, $\cos\Phi_i$, $\eta_i$, and bending angle values, which are considered more physically meaningful and compact. Finally, all input features are standardized using z-score normalization. The standardization parameters, mean and standard deviation, are computed exclusively from the training set and subsequently applied to both training and test sets. This ensures consistent input distributions across model evaluations and prevents data leakage during the learning process.

\subsection{Graph Construction Strategies}
\label{graph_construction}
In this study, four distinct graph construction strategies were explored to structure the data for the downstream task of predicting \( p_T \). The goal was to leverage various representations of the data by encoding the detector's spatial and physical features as graph structures. Each method utilized different nodes, edge definitions, and edge attributes to explore the impact of different graph configurations on the model's ability to predict \( p_T \). The following sections provide detailed descriptions of each graph construction method.

\subsubsection{Method 1: Station as a Node Graph Representation}
\paragraph{Conceptual Basis}
In this method, we model the CMS detector's geometry and particle trajectory using a graph representation. Each of the four sequential muon detection layers (henceforth referred to as Station 0, Station 1, Station 2, and Station 3) encountered by a particle is represented as a node. The goal is to leverage the sequential nature of these measurements. Each node's features encapsulate the localized kinematic and geometric measurements recorded at the corresponding station, enabling the GNN to capture correlations and dependencies across different detector layers, ultimately modeling the particle's progression through the detector.

\paragraph{Graph Construction}
For each particle track sample \( k \), a graph \( G_k = (V, E, X_k) \) is constructed. The node set \( V \) is fixed with four nodes, where \( s_i \in V \) corresponds to Station \( i \):
\begin{equation}
V = \{s_0, s_1, s_2, s_3\}
\end{equation}
The node features are derived from the preprocessed feature vector \( f_k \in \mathbb{R}^{28} \) for the \( k \)-th track. These features are reshaped into a node feature matrix \( X_k \in \mathbb{R}^{4 \times 7} \), where each row \( x_{k, s_i} \in \mathbb{R}^7 \) corresponds to the feature vector for station \( s_i \). The mapping from the 1D vector \( f_k \) to the 2D matrix \( X_k \) is given by:
\begin{equation}
X_k[i, j] = f_k[j \cdot N_{st} + i]
\end{equation}
where \( N_{st} = 4 \) denotes the number of stations, and \( i \in \{0, 1, 2, 3\} \), \( j \in \{0, \dots, 6\} \). The 7 features per station include angular measurements (e.g., \( \phi, \theta \)) and layer identifiers. The edges \( E \) are statically defined, forming an undirected linear chain between neighboring stations:
\begin{equation}
E = \{(s_i, s_{i+1}) \mid i \in \{0, 1, 2\} \} \cup \{(s_{i+1}, s_i) \mid i \in \{0, 1, 2\}\}
\end{equation}
Explicitly, the edge list, using the 0-indexed node identifiers, is:
\begin{equation}
E = \{(0, 1), (1, 2), (2, 3), (3, 2), (2, 1), (1, 0)\}
\end{equation}
This bidirectional structure allows message passing between neighboring stations, facilitating the flow of information across layers of the detector.

\paragraph{Rationale for Graph Structure}
The sequential graph topology reflects the CMS detector's organization, where particles move through distinct spatial layers. The bidirectional edges enable the GNN to learn from past and future measurements, improving feature propagation along the particle's trajectory. This design is similar to methods for modeling temporal or spatial dependencies in sequential data \cite{ruiz2020gated}. Still, our approach uniquely integrates domain-specific detector features, such as angular measurements and layer identifiers, into the graph structure.

\subsubsection{Method 2: Feature as Node Graph Representation}

\paragraph{Conceptual Basis}
This method redefines particle trajectory analysis by conceptualizing physical features, such as kinematic quantities (e.g., \(\Phi\), \(\eta\), bending angle), as nodes in a graph. These features are tracked across four detector stations, allowing the model to capture both the temporal evolution of these measurements and the interdependencies between different features. The seven nodes in the graph represent seven predefined physical quantities: \(\Phi\), \(\eta\), bending angle, and their respective differentials (\(\Delta \Phi\), \(\Delta \eta\), \(\Delta R\), and charge \(q\)). These features are engineered from the original 28 input features, excluding the global road features, where values for each feature across the four stations become the 4D feature vector for each corresponding node. Specifically, for each of the seven features, we extract the values across the four stations and assign them as the components of the respective node's feature vector. The feature vector at each station is derived from the corresponding physical quantity measurements, forming a clear mapping from the original 28 features to the seven predefined ones.

\paragraph{Graph Construction}
For each particle track \(k\), a graph \(G_k = (V, E, X_k)\) is constructed. The node set \(V\) consists of seven nodes, each corresponding to one of the predefined feature types. These feature types include \(\Phi\), \(\eta\), bending angle, and their respective differentials (\(\Delta \Phi\), \(\Delta \eta\), \(\Delta R\), and charge \(q\)). The corresponding feature values across the four stations are encoded within 4-dimensional feature vectors, which form the components of each node's feature vector. To construct the node feature matrix \(X_k'\), we first reshape the 28D feature vector \(f_k \in \mathbb{R}^{28}\) into an intermediate matrix \(M_k \in \mathbb{R}^{4 \times 7}\), where each row represents the four station measurements of a feature type, and each column corresponds to one of the feature types across the stations. After reshaping, the matrix \(M_k\) is transposed to form the final node feature matrix:
\begin{equation}
X_k' = M_k^T \in \mathbb{R}^{7 \times 4}
\end{equation}
where each row of \(X_k'\) corresponds to a feature type with its values measured across the four stations. The edge set \(E\) includes both temporal edges and inter-feature correlation edges. Temporal edges model the progression of each feature across stations, while inter-feature edges capture the relationships between different feature types. The temporal edges for a given feature \(f^t_j\) are represented as:
\begin{equation}
E_{temporal} = \{ (f^t_i, f^t_{i+1}) \mid i \in \{0, 1, 2\} \}  \end{equation}
In addition to these, inter-feature edges are defined to model correlations between different features. These inter-feature edges connect different feature types, with a hub-like structure around \(f^t_2\) (representing one of the physical quantities, not the station itself). The explicit edge list is given by:
\begin{equation}
E = \{(0,1), (1,0), (1,2), (2,1), (2,3), (3,2), (0,2), (2,0), (2,4), (4,2), (2,5), (5,2), (2,6), (6,2)\}    
\end{equation}
This edge structure reflects both the temporal progression of features and the inter-feature relationships, enabling the GNN to model how features interact over time and across different stations.

\paragraph{Rationale for Graph Structure}
This graph structure supports three key learning mechanisms. First, vertical propagation models the evolution of feature values across different detector layers. For each node \(f^t_j\), the GCN learns how its feature evolves as the particle progresses through the stations. This process can be formalized using the graph convolution operation:
\begin{equation}
h_k(l) = \sigma(A X_k^{(l-1)} W^{(l)})    
\end{equation}
where \(h_k(l)\) represents the hidden state of the nodes at layer \(l\), \(A\) is the adjacency matrix (which includes edges), \(W^{(l)}\) is the weight matrix for layer \(l\), and \(\sigma\) is the activation function applied at each layer. The second mechanism, horizontal correlation, allows the model to learn the relationships between different feature types, such as the complementarity between \(\eta\) and \(\Phi\). This is achieved through the inter-feature edges that capture the geometric and physical relationships between feature nodes. The GNN will learn how these features interact with each other during the particle's progression. The third mechanism, centralized attention, prioritizes important features, particularly those that have a significant role in the reconstruction process. Specifically, the node \(f^t_2\), which represents a key feature, serves as the hub in the graph. The attention mechanism within the model gives more weight to the features connected to this node. This is not necessarily an explicit attention mechanism but rather an emergent property from the GCN structure and its hub-and-spoke connectivity. The influence of \(f^t_2\) on other features is hypothesized to result from the graph's edge structure, which allows for greater message passing from this node, thereby leading to its increased importance.

\subsubsection{Method 3: Bending Angle-Centric Graph Representation}

\paragraph{Conceptual Basis}
This method centers the graph representation on the bending angle, a pivotal observable for inferring particle momentum within a magnetic field. Each of the four selected detector stations (Stations 0, 2, 3, and 4, chosen for their significance in capturing track curvature) is represented as a distinct node. The graph is engineered to model how localized magnetic deflections, encapsulated by bending angle measurements and associated kinematic features at each station, propagate and interact. This approach aims to capture the relationship between bending geometry and particle momentum through three primary mechanisms: (i) node-level dynamics, focusing on comprehensive per-station measurements including bending angle; (ii) edge-level angular coherence, encoding the evolution and covariance of azimuthal (\(\Phi\)) and (\(\eta\)) coordinates between stations; and (iii) global trajectory curvature, learned via message passing across all interconnected station nodes to model compound bending effects.

\paragraph{Graph Construction}
For each particle track sample \(k\) from the \(p_T\) outlier-filtered dataset (1,029,592 samples, as detailed in Section~\ref{sec:preprocessing}), a graph \(G_k = (V, E, X_k, E_{attr})\) is constructed. The node set \(V\) is fixed, comprising four nodes, \(V = \{v_0, v_2, v_3, v_4\}\), where each node \(v_i \in V\) corresponds to detector station \(i \in \{0, 2, 3, 4\}\). Each such node \(v_i \in V\) is assigned a 4-dimensional feature vector \(x_{k,v_i} \in \mathbb{R}^4\). These features are extracted from the 16 selected, standardized, and pre-engineered quantities (see Section 3.2) available for each station, specifically:
\begin{equation}
 x_{k,v_i} = [BA_i^{std},~\sin \phi_i^{std},~\cos \phi_i^{std},~\eta_i^{std}]   
\end{equation}
In this vector, \(BA_i^{std}\) denotes the standardized bending angle, while \(\sin \phi_i^{std}\), \(\cos \phi_i^{std}\), and \(\eta_i^{std}\) represent the standardized sine of azimuthal angle, cosine of azimuthal angle, and $\eta$ for station \(i\), respectively. The complete node feature matrix \(X_k \in \mathbb{R}^{4 \times 4}\) is formed by concatenating these individual node feature vectors.

The edge set \(E\) defines a fully connected topology, excluding self-loops, among these four station nodes. Consequently, a directed edge \((v_i, v_j)\) exists for all pairs of distinct stations \(i, j \in \{0, 2, 3, 4\}\), resulting in 12 directed edges per graph. This structure facilitates direct information exchange between all selected station nodes. Furthermore, each directed edge \((v_i, v_j) \in E\) is augmented with an edge feature vector \(e_{ij} \in \mathbb{R}^3\). These attributes encode the relative geometric displacement between the source station \(i\) and the target station \(j\), calculated using the standardized, pre-engineered features:
\begin{equation}
e_{ij} = \left[ \sin \phi_j^{\text{std}} - \sin \phi_i^{\text{std}},~\cos \phi_j^{\text{std}} - \cos \phi_i^{\text{std}},~\eta_j^{\text{std}} - \eta_i^{\text{std}} \right]
\end{equation}
The collection of these vectors constitutes the edge attribute tensor \(E_{attr}\).

\paragraph{Rationale for Graph Structure}
This design provides the GNN with rich, localized information at each node (bending angle plus angular context) and explicit relational information about inter-station geometry via edge attributes. Full connectivity enables the model to learn complex, non-local dependencies by capturing how features at one station influence or correlate with features at any other station, conditioned by their spatial relationship. The message passing mechanism, operating on these inputs, is crucial for this. For instance, a message \(m_{ij}^{(l)}\) passed from node \(v_j\) to node \(v_i\) at layer \(l\) can be conceptualized as a function of their hidden states and edge attributes:
\begin{equation}
 m_{ij}^{(l)} = \text{MLP}_{message}(h_{k,v_i}^{(l)} \oplus h_{k,v_j}^{(l)} \oplus e_{ij})   
\end{equation}
where \(h_{k,v_i}^{(l)}\) is the hidden state of node \(v_i\) and \(\oplus\) denotes concatenation. This allows the network to model momentum-geometry coupling effectively.

\subsubsection{Method 4: $\eta$-Centric Geometric Graph Representation}

\paragraph{Conceptual Basis}
This fourth method reorients the graph representation to be \(\eta\)-centric, leveraging the properties of \(\eta\), its direct relation to the particle's polar angle (\(\theta\)) and its invariance under longitudinal Lorentz boosts, as a primary geometric observable for modeling particle trajectories. Within this framework, nodes encode \(\eta\) measurements and associated kinematic features across selected detector stations (Stations 0, 2, 3, and 4). In contrast to previous methods, the edges are constructed dynamically based on proximity in the \(\eta - \phi\) space, reflecting the geometric locality often crucial in particle physics analyses. This "\(\eta\)-centric" approach aims to explicitly capture: (i) longitudinal trajectory development, by tracking \(\eta\) and its local gradients across stations; (ii) azimuthal-rapidity correlations, through features and dynamically-formed edges that consider both \(\eta\) and \(\phi\); and (iii) projective geometry consistency, by leveraging features sensitive to the inherent structure of particle paths in $\eta$, which are fundamentally linked to Lorentz transformation properties.

\paragraph{Graph Construction}
For each particle track sample \(k\) from the \(p_T\) outlier-filtered dataset (1,029,592 samples, derived from the CMS Trigger Dataset as detailed in Sections 3.1.1 and 3.2), a graph \(G_k = (V, E_k, X_k, E_{attr,k})\) is constructed, where the edge set \(E_k\) and its attributes \(E_{attr,k}\) can be sample-dependent. The node set \(V\) is fixed, comprising four nodes, \(V = \{v_0^{\eta}, v_2^{\eta}, v_3^{\eta}, v_4^{\eta}\}\), where each node \(v_i^{\eta} \in V\) corresponds to detector station \(i \in \{0, 2, 3, 4\}\). Each node \(v_i^{\eta} \in V\) is assigned a 5-dimensional feature vector \(x_{k,v_i^{\eta}} \in \mathbb{R}^5\), derived from standardized, pre-engineered quantities:
\begin{equation}
   x_{k,v_i^{\eta}} = [\eta_i^{std},~\sin \phi_i^{std},~\cos \phi_i^{std},~\Delta \eta_i^{std},~BA_i^{std}] 
\end{equation}
Here, \(\eta_i^{std}\) is the standardized $\eta$, \(\sin \phi_i^{std}\) and \(\cos \phi_i^{std}\) are the standardized sine and cosine of the azimuthal angle, \(BA_i^{std}\) is the standardized bending angle, and \(\Delta \eta_i^{std}\) is the standardized local gradient of $\eta$ for station \(i\). The term \(\Delta \eta_i\) is calculated as \(\eta_i - \eta_{i_p}\), where \(i_p\) is the index of the station immediately preceding \(i\) in the physical sequence \(\{0, 2, 3, 4\}\); for the first station in this sequence (station 0), \(\Delta \eta_0\) is taken as \(\eta_0\) itself. The complete node feature matrix \(X_k \in \mathbb{R}^{4 \times 5}\) concatenates these vectors.

The edge set \(E_k\) for each graph \(G_k\) is constructed dynamically using a \(k\)-Nearest Neighbors (k-NN) algorithm, with \(k = 3\). Edges connect nodes based on their proximity in the \(\eta - \phi\) plane. Specifically, a directed edge \((v_i^{\eta}, v_j^{\eta})\) exists if node \(v_j^{\eta}\) is one of the \(k = 3\) nearest neighbors of \(v_i^{\eta}\) (excluding \(v_i^{\eta}\) itself). Proximity is measured by the Euclidean distance:
\begin{equation}
   \Delta R_{ij} = (\eta_i - \eta_j)^2 + (\Delta \phi_{ij})^2 
\end{equation}
where \(\Delta \phi_{ij}\) is the minimal difference between \(\phi_i\) and \(\phi_j\) accounting for \(2\pi\) periodicity. Each dynamically formed directed edge \((v_i^{\eta}, v_j^{\eta}) \in E_k\) is associated with a feature vector \(e_{k,ij} \in \mathbb{R}^3\), capturing relative kinematics:
\begin{equation}
 e_{k,ij} = [\eta_j^{std} - \eta_i^{std},~\Delta \phi_{ij}^{std},~(\Delta R_{ij}^2)^{std}]   
\end{equation}
Here, \(\Delta \phi_{ij}^{std}\) is the standardized minimal azimuthal angle difference, and \((\Delta R_{ij}^2)^{std}\) is the standardized squared Euclidean distance in the \(\eta - \phi\) space. These vectors form the edge attribute tensor \(E_{attr,k}\). The resulting graph has 12 directed edges, and thus the edge index tensor has a shape of \([2, 12]\) and the edge attribute tensor has a shape of \([12, 3]\).

\paragraph{Rationale for Graph Structure}
This adaptive graph structure is designed to leverage key aspects of \(\eta\)-dominated particle tracking. The inclusion of the local gradient \(\Delta \eta_i\) within node features is motivated by the desire to model longitudinal consistency in track development smoothly; this feature directly informs the GNN about local changes in $\eta$, akin to promoting smoothness implicitly. The kNN-based edge construction enforces geometric locality, prioritizing interactions between stations that are close in the \(\eta - \phi\) phase space, which is often critical for resolving nearby particles or analyzing jet substructure. Finally, the inclusion of bending angles (\(BA_i\)) as secondary node features ensures that crucial magnetic coupling information, indicative of momentum, is preserved and can be correlated by the GNN with the primary \(\eta\)-based features, even without explicit edges solely dedicated to BA relationships.

\subsection{Model Architecture}
\label{model_architecture}
We evaluated various GNN backbones, GCN, GAT, GraphSAGE~\cite{graphsage}, EdgeConv~\cite{edgeconv}, and our proposed MPL, across four graph construction methods. In MPL, we experimented with different combinations of graph construction strategies, loss functions, edge definitions, and embedding dimensions, which affected both model size and performance. After message passing, global pooling was applied to obtain graph-level embeddings, followed by an MLP regression head for predicting $p_T$.

\subsubsection{MPL Architecture}
The MPL is our novel GNN operator designed to capture node-wise and edge-wise interactions in irregular jet-based particle data. Its rationale stems from the limitations of standard GNN aggregators, which often overlook fine-grained edge attributes and dynamically modulated interactions. To address this, MPL explicitly encodes edge features, models directional dependencies via $(\mathbf{h}_j - \mathbf{h}_i)$, and applies an attention mechanism to weight the importance of each message based on source and target node contexts. The design is tailored for physical problems, such as $p_T$ regression, where inter-particle relations (e.g., distance, charge, energy flow) must be captured in both feature space and topology.

Let $\mathbf{h}_i \in \mathbb{R}^{F_{\text{in}}}$ denote the feature vector of node $i$, and let $\mathbf{e}_{ij} \in \mathbb{R}^{d_e}$ represent the edge attributes from node $j$ to node $i$. Edge attributes are first transformed via a two-layer MLP:
\begin{equation}
\tilde{\mathbf{e}}_{ij} = \text{MLP}_{\text{edge}}(\mathbf{e}_{ij}) = \text{ReLU}(\mathbf{W}_2 \cdot \text{ReLU}(\mathbf{W}_1 \cdot \mathbf{e}_{ij})),
\end{equation}
where $\mathbf{W}_1, \mathbf{W}_2$ are learnable parameters of the edge encoder. For each edge $(j \to i)$, the input to the message function concatenates $\mathbf{h}_i$, the difference $\mathbf{h}_j - \mathbf{h}_i$, and the transformed edge attribute $\tilde{\mathbf{e}}_{ij}$, i.e., $\mathbf{z}_{ij} = [\mathbf{h}_i, \mathbf{h}_j - \mathbf{h}_i, \tilde{\mathbf{e}}_{ij}]$. This is passed through a feedforward layer with ReLU: $\mathbf{m}'_{j \to i} = \text{ReLU}(\text{MLP}_1(\mathbf{z}_{ij}))$. The message is modulated using a learned attention coefficient: $\mathbf{a}_1 = \tanh(\text{MLP}_5(\mathbf{h}_i)), \quad
\mathbf{a}_2 = \tanh(\text{MLP}_6(\mathbf{m}'_{j \to i}))$,
and $\alpha_{j \to i} = \text{softmax}_j(\text{MLP}_7(\mathbf{a}_1 \odot \mathbf{a}_2))$, where $\odot$ denotes element-wise multiplication, and the softmax is computed over incoming edges to node $i$. The final weighted message is then $\mathbf{m}_{j \to i} = \alpha_{j \to i} \mathbf{m}'_{j \to i}$. 

All messages are aggregated using a permutation-invariant summation operator $\bigoplus$, yielding 
$\mathbf{M}_i = \bigoplus_{j \in \mathcal{N}(i)} \mathbf{m}_{j \to i}$. The node's self-representation is transformed as $\tilde{\mathbf{h}}_i = \text{ReLU}(\text{MLP}_2(\mathbf{h}_i))$. Two sigmoid gates determine the contribution of the message and the self-feature:
$\mathbf{w}_1 = \sigma(\text{MLP}_3([\tilde{\mathbf{h}}_i, \mathbf{M}_i]))$ and $\mathbf{w}_2 = \sigma(\text{MLP}_4([\tilde{\mathbf{h}}_i, \mathbf{M}_i]))$. The final updated representation is computed as $\mathbf{h}'_i = \mathbf{w}_1 \odot \mathbf{M}_i + \mathbf{w}_2 \odot \tilde{\mathbf{h}}_i$. This update rule allows the model to adaptively decide how much new information from neighbors should be integrated, versus retaining its own transformed state, leading to more expressive and stable updates, especially in sparse particle graphs.

\subsection{Loss Functions Overview}
\label{loss_function}
In this work, we use a combination of domain-informed and standard loss functions to optimize the model's prediction of \( p_T \). The primary objective is to minimize the error between the predicted and true \( p_T \) values, while incorporating domain knowledge about particle physics and momentum behavior.

\subsubsection{Mean Squared Error (MSE) Loss}
The Mean Squared Error (MSE) loss is used as the core objective function. For a batch of \( M \) samples, the MSE loss \( L_{\text{MSE}} \) is computed as follows:
\begin{equation}
L_{\text{MSE}} = \frac{1}{n} \sum_{m=1}^{N} \left( p_T(N) - \hat{p}_T(m) \right)^2
\end{equation}
where \( p_T(m) \) and \( \hat{p}_T(m) \) are the true and predicted transverse momenta for the \( m \)-th sample, respectively.

\subsubsection{Domain-Informed $p_T$ Loss}
To ensure the model performs well across various momentum ranges, we introduce a domain-informed loss function that emphasizes critical regions of \( p_T \). This weighted loss function, known as the $p_T$ Loss, applies different weights to different momentum ranges to focus on more challenging regions, such as low and high \( p_T \) values. The loss is computed as:
\begin{equation}\label{eq:pt_1}
L = \frac{1}{C \cdot n} \sum_{m=1}^{n} W(p_T(m)) \left( p_T(m) - \hat{p}_T(m) \right)^2
\end{equation}
where \( C = 250 \) is a scaling constant, and \( W(p_T) \) is a piecewise weight function that assigns varying importance to \( p_T \) values across different ranges:
\begin{equation}\label{eq:pt_2}
W(p_T) =
\begin{cases}
p_T & \text{if } p_T < 80 \, \text{GeV}/c \\
2.4 & \text{if } 80 \leq p_T < 160 \, \text{GeV}/c \\
2.4 + 10 & \text{if } p_T \geq 160 \, \text{GeV}/c
\end{cases}
\end{equation}
This piecewise weighting function prioritizes lower \( p_T \) values to improve background rejection and penalizes higher \( p_T \) values to address potential saturation effects, ensuring effective learning across all momentum ranges. Formal analysis of $p_T$ loss can be found in Section~\ref{sec:formal_analysis}, Proposition~\ref{prop1}.

\subsubsection{Custom \(p_T\) Loss with Asymmetric Penalty}
A specialized "Custom \(p_T\) Loss" is also investigated. Let \(p_T^{(m)}\) be the true target transverse momentum for sample \(m\), and let \(\hat{p_T}^{(m)}\) be the model's output value for \(p_T\), which is explicitly clipped to be no less than a predefined lower $p_T$ limit (LPL) before being used in this loss calculation. The Custom \(p_T\) Loss for a batch of \(N\) samples is defined as:
\begin{equation}\label{eq:custom_1}
\begin{split}
L_{\text{Custom}} = \frac{1}{N} \sum_{m=1}^{N} \Biggl[ &\left(p_{T}^{(m)} - \hat{p_T}^{(m)}\right)^2 + \Biggr( \mathbb{1}[\hat{p_T}^{(m)} > \text{LPL}] \left( \frac{1}{1 + e^{-3(\hat{p_T}^{(m)} - \text{LPL})}} - 1 \right) \\
&+ \mathbb{1}[\hat{p_T}^{(m)} \le \text{LPL}] \left( -\frac{1}{2} \right) \Biggr) \Biggr]
\end{split}
\end{equation}
where \( \mathbb{1}(\cdot) \) is the indicator function (1 if the condition is true, 0 otherwise). The first term is the squared error between the true \(p_T\) and the (already clipped) prediction \(\hat{p_T}^{(m)}\). The subsequent terms apply asymmetric penalties: a sigmoid-based penalty when the (clipped) prediction is above LPL, and a fixed penalty when it is at or below LPL. This structure aims to regularize predictions, particularly around the specified lower \(p_T\) threshold. Formal analysis can be found in Section \ref{sec:formal_analysis}, Proposition \ref{prop2}.

\subsection{Formal Analysis of Custom Loss Functions}
\label{sec:formal_analysis}
In this section, we provide a more formal analysis of the properties of our custom loss functions. These propositions are not theorems of global model optimality but rather aim to rigorously demonstrate the intended behavior and mathematical characteristics of these specific components, which contribute to the overall learning dynamics for $p_T$ estimation.

\begin{proposition}\label{prop1}
Properties of the Domain-Informed $p_T$ Loss.
The Domain-Informed $p_T$ Loss, defined by Equations (\ref{eq:pt_1}) and (\ref{eq:pt_2}) in the main paper, modulates the learning gradient for each sample based on its true transverse momentum $\pttrue{m}$. Specifically:
\begin{enumerate}
    \item[(a)] For $\pttrue{m} < 80 \text{ GeV/c}$, the effective learning rate for a sample's squared error is proportional to $\pttrue{m}$.
    \item[(b)] For $80 \text{ GeV/c} \le \pttrue{m} < 160 \text{ GeV/c}$, the effective learning rate for a sample's squared error is a constant $2.4/C$.
    \item[(c)] For $\pttrue{m} \ge 160 \text{ GeV/c}$, the effective learning rate for a sample's squared error is a constant $12.4/C$.
\end{enumerate}
This results in a non-uniform emphasis on different $p_T$ regimes, with a distinct weighting profile across the momentum spectrum, including a significant drop in weight when transitioning from just below $80 \text{ GeV/c}$ to $80 \text{ GeV/c}$.
\end{proposition}

\begin{proof}[Demonstration]
The Domain-Informed $p_T$ Loss for a batch of $N$ samples is given by (referring to Eq. (\ref{eq:pt_1}) from the main paper):
\[ L_{p_T\text{-informed}} = \frac{1}{C \cdot N} \sum_{m=1}^{N} W(\pttrue{m}) (\pttrue{m} - \pthat{m})^2 \]
where $C=250$ is a scaling constant (as per line 329 of the main paper), and $W(p_T)$ is defined as (referring to Eq. (\ref{eq:pt_2}) from the main paper):
\[ W(p_T) = \begin{cases} p_T & \text{if } p_T < 80 \text{ GeV/c} \\ 2.4 & \text{if } 80 \text{ GeV/c} \le p_T < 160 \text{ GeV/c} \\ 2.4 + 10 & \text{if } p_T \ge 160 \text{ GeV/c} \end{cases} \]
Consider the gradient of this loss with respect to a single predicted value $\pthat{k}$ for sample $k$:
\[ \frac{\partial L_{p_T\text{-informed}}}{\partial \pthat{k}} = \frac{1}{C \cdot N} W(\pttrue{k}) \cdot 2 (\pttrue{k} - \pthat{k}) \cdot (-1) = -\frac{2}{C \cdot N} W(\pttrue{k}) (\pttrue{k} - \pthat{k}) \]
The term $W(\pttrue{k})$ acts as a multiplier on the error signal $(\pttrue{k} - \pthat{k})$.

(a) If $\pttrue{k} < 80 \text{ GeV/c}$, then $W(\pttrue{k}) = \pttrue{k}$. The gradient magnitude is scaled by $\pttrue{k}$. This implies that errors for samples with $p_T$ closer to $80 \text{ GeV/c}$ (e.g., $p_T=79 \text{ GeV/c}$) will result in a larger gradient update (scaled by 79) compared to errors for samples with very low $p_T$ (e.g., $p_T=10 \text{ GeV/c}$, scaled by 10), assuming similar error magnitudes.

(b) If $80 \text{ GeV/c} \le \pttrue{k} < 160 \text{ GeV/c}$, then $W(\pttrue{k}) = 2.4$. The gradient magnitude is scaled by a constant $2.4$.

(c) If $\pttrue{k} \ge 160 \text{ GeV/c}$, then $W(\pttrue{k}) = 12.4$. The gradient magnitude is scaled by a constant $12.4$.

The transition in weighting is notable:
\begin{itemize}
    \item At $\pttrue{k} = 79.9 \text{ GeV/c}$ (approaching from below), $W(\pttrue{k}) \approx 79.9$.
    \item At $\pttrue{k} = 80.0 \text{ GeV/c}$, $W(\pttrue{k}) = 2.4$.
\end{itemize}
This represents a sharp decrease in the weight applied to the squared error as $p_T$ crosses the $80 \text{ GeV/c}$ threshold from below. The paper's description (line 331 of the main paper) suggests this prioritizes lower $p_T$ values for background rejection. The $W(p_T)=p_T$ term for $p_T<80$ gives more weight to samples closer to the 80 GeV/c decision boundary within that low-$p_T$ range. The relatively low constant weight (2.4) in the 80-160 GeV/c range might reflect a region where $p_T$ resolution is intrinsically better or less critical than the extremes. The higher weight (12.4) for $p_T \ge 160 \text{ GeV/c}$ emphasizes accuracy for high-momentum particles, addressing potential saturation effects. The scaling constant $C=250$ normalizes the overall loss magnitude.
\end{proof}

\begin{proposition}\label{prop2}
Properties of the Custom $p_T$ Loss with Asymmetric Penalty.
The Custom $p_T$ Loss with Asymmetric Penalty (Eq. (\ref{eq:custom_1}) in the main paper), for a prediction $\pthat{m}$ (assumed to be the model's direct output before explicit hard clipping for the penalty term's $LPL$), is designed to:
\begin{enumerate}
    \item[(a)] Add a fixed penalty of $\frac{1}{2}$ to the loss if $\pthat{m} \le LPL$.
    \item[(a)] Add a variable term $g(\pthat{m}) = \frac{1}{1+e^{-3(\pthat{m} - LPL)}} - 1$ if $\pthat{m} > LPL$. This term $g(\pthat{m})$ ranges from approximately $-\frac{1}{2}$ (for $\pthat{m} \to LPL^+$) to $0$ (for $\pthat{m} \to \infty$), effectively acting as a ``bonus'' (loss reduction) that is largest when $\pthat{m}$ is just above $LPL$ and diminishes as $\pthat{m}$ increases further.
    \item[(c)] Consequently, there is a sharp increase (discontinuity of magnitude 1) in the penalty component of the loss as $\pthat{m}$ crosses $LPL$ from above to below, strongly discouraging predictions at or below $LPL$.
\end{enumerate}
\end{proposition}

\begin{proof}[Demonstration]
Let the prediction from the model be $\pthat{m}$. The paper states (line 336 of the main paper) that the output value for $p_T$ is ``explicitly clipped to be no less than a predefined LPL before being used in this loss calculation.'' Let $\hat{p}_{T,clip}^{(m)}$ denote this value used in the squared error term. For the penalty terms, let's consider the behavior based on the model's output $\pthat{m}$ relative to the threshold $LPL$.
The penalty component $P(\pthat{m})$ of the loss function in Eq. (\ref{eq:custom_1}) of the main paper is:
\[ P(\pthat{m}) = \indicator{\pthat{m} > LPL} \left( \frac{1}{1+e^{-3(\pthat{m} - LPL)}} - 1 \right) + \indicator{\pthat{m} \le LPL} \left( \frac{1}{2} \right) \]

(a) If $\pthat{m} \le LPL$:
The second term applies, and the first term is zero.
$P(\pthat{m}) = \frac{1}{2}$.
This adds a fixed positive value of $\frac{1}{2}$ to the squared error component of the loss.

(b) If $\pthat{m} > LPL$:
The first term applies, and the second term is zero. Let $g(x) = \frac{1}{1+e^{-3(x - LPL)}} - 1$.
So, $P(\pthat{m}) = g(\pthat{m})$.
Let $y = \pthat{m} - LPL$. Since $\pthat{m} > LPL$, $y > 0$.
The term $g(\pthat{m})$ can be rewritten as $\sigma(3y) - 1$, where $\sigma(z) = \frac{1}{(1+e^{-z})}$ is the sigmoid function.
As $y \to 0^+$ (i.e., $\pthat{m} \to LPL^+$), $\sigma(3y) \to \sigma(0) = \frac{1}{2}$. So, $g(\pthat{m}) \to \frac{1}{2} - 1 = -\frac{1}{2}$.
As $y \to \infty$ (i.e., $\pthat{m} \to \infty$), $\sigma(3y) \to 1$. So, $g(\pthat{m}) \to 1 - 1 = 0$.
Thus, for $\pthat{m} > LPL$, $P(\pthat{m})$ ranges in $[-\frac{1}{2}, 0)$. This term is always non-positive, so it reduces the overall loss contribution from the squared error. This reduction is maximal (value of $-\frac{1}{2}$) when $\pthat{m}$ is just above $LPL$, and the reduction diminishes (approaches $0$) as $\pthat{m}$ becomes much larger than $LPL$.

(c) Discontinuity at $LPL$:
We evaluate the limit of $P(\pthat{m})$ as $\pthat{m}$ approaches $LPL$ from both sides:
\[ \lim_{\pthat{m} \to LPL^+} P(\pthat{m}) = -\frac{1}{2} \]
\[ \lim_{\pthat{m} \to LPL^-} P(\pthat{m}) = \frac{1}{2} \]
And $P(LPL) = \frac{1}{2}$.
The jump in the penalty at $\pthat{m} = LPL$ is $P(LPL) - \lim_{\pthat{m} \to LPL^+} P(\pthat{m}) = \frac{1}{2} - (-\frac{1}{2}) = 1$.
This discontinuity creates a strong incentive for the model to predict values $\pthat{m} > LPL$. Once above $LPL$, the ``bonus'' term $g(\pthat{m})$ further reduces the loss, most significantly when just above $LPL$. The diminishing nature of this bonus for much larger $\pthat{m}$ prevents this term from overly encouraging arbitrarily high predictions if the squared error term would otherwise penalize them. This focuses the regularization effect of the penalty terms around the $LPL$ threshold.
\end{proof}

\begin{table}[!ht]
\caption{Comparative results for $p_T$ estimation using GNNs and other DL models. \ul{Underline} indicates the best-performing DL model among six baselines. \textbf{Bold} indicates models that outperform the strong TabNet baseline by achieving both a lower MAE (in GeV/c) and fewer parameters. MAE is reported as mean ± standard deviation, calculated from the median-accuracy run across 3 independent training instances to account for variability due to random initialization and optimization dynamics.}
\label{tab:my-table}
\resizebox{\columnwidth}{!}{%
\begin{tabular}{@{}lcccccccc@{}}
\toprule[1.5pt]
  \textbf{Modeling Approach} &
  \textbf{Edges} &
  \textbf{Model Backbone} &
  \textbf{Loss Function} &
  \textbf{Loss (\textcolor{myred}{$\downarrow$})} &
  \textbf{MAE (\textcolor{myred}{$\downarrow$})} &
  \textbf{\# Params (\textcolor{myred}{$\downarrow$})} \\ \midrule[1pt]
TabNet~\cite{tabnet_2021} &
  - &
  - &
  MSE &
  \multicolumn{1}{c}{2.9746} &
  \ul{0.960700$\pm$0.015331} &
  \ul{6696} \\ 
FCNN &
  - &
  5 Hidden FCN &
  MSE &
  \multicolumn{1}{c}{18647.681641} &
  19.361114$\pm$0.142454 &
  213761 \\ 
CNN &
  - &
  6 Conv2D &
  MSE &
  \multicolumn{1}{c}{18788.152344} &
  18.745565$\pm$0.157818 &
  1050305 \\ 
CNN-Grid &
  - &
  3 Conv2D, 3 FCN &
  MSE &
  \multicolumn{1}{c}{18777.416016} &
  20.612196$\pm$0.190880 &
  230593 \\ 
Dual Channel CNN-FCNN &
  - &
  3 Conv2D (Channel-1), 4 FCN (Channel-2) &
  MSE &
  \multicolumn{1}{c}{18771.123047} &
  21.929241$\pm$0.404957 &
  52193 \\ 
LSTM &
  - &
  2 LSTM, 1 FCN &
  MSE &
  \multicolumn{1}{c}{17629.876953} &
  16.688435$\pm$0.154408 &
  32257 \\
  \midrule
\multirow{10}{*}{GNN (each station as a node)}&
  Sequential ($0 \rightarrow 1 \rightarrow 2 \rightarrow 3$) &
  4 GCN (embed dim: 128) &
  $p_T$ &
  3.83922549 &
  43.210644$\pm$0.979542 &
  55460 \\
 &
  Sequential ($0 \rightarrow 1 \rightarrow 2 \rightarrow 3$) &
  GCN, 2 GAT, SAGE (embed dim: 128) &
  $p_T$ &
  0.00066572 &
  25.429537$\pm$0.431732 &
  59812 \\
 &
  Fully-Connected &
  4 MPL (embed dim: 128) &
  $p_T$ &
  1.36491351 &
  15.822893$\pm$0.436666 &
  101152 \\
 &
  Fully-Connected &
  4 MPL (embed dim: 64) &
  $p_T$ &
  1.338410 &
  16.227703$\pm$0.095143 &
  101487 \\
 &
  Fully-Connected &
  4 MPL (embed dim: 128) &
  MSE &
  0.00047896 &
  52.640570$\pm$0.860296 &
  101152 \\
 &
  Fully-Connected &
  4 MPL (embed dim: 128) &
  MSE &
  0.00263888 &
  13.529569$\pm$0.064004 &
  101152 \\
 &
  Sequential ($0 \rightarrow 1 \rightarrow 2 \rightarrow 3$) &
  4 MPL (embed dim: 128) &
  MSE &
  0.00191543 &
  13.303385$\pm$0.044040 &
  101152 \\
 &
  Fully-Connected &
  4 MPL (embed dim: 128) &
  $p_T$ &
  0.00416450 &
  2.071646$\pm$0.044040 &
  101152 \\
 &
  Fully-Connected &
  4 MPL (embed dim: 128) &
  Custom $p_T$ &
  0.46798712 &
  0.817826$\pm$0.016848 &
  101152 \\
 &
  Fully-Connected &
  4 MPL (embed dim: 128) &
  Custom $p_T$ &
  0.4529373 &
  0.812923$\pm$0.007707 &
  101152 \\
 &
  Fully-Connected &
  4 MPL (embed dim: 128) &
  Custom $p_T$ &
  0.48763093 &
  0.892330$\pm$0.023982 &
  101152 \\
 &
  Fully-Connected &
  2 EdgeConv (embed dim: 64) &
  MSE &
  0.11814979 &
  0.821690$\pm$0.022785 &
  19649 \\
 &
  Fully-Connected &
  4 EdgeConv (embed dim: 16) &
  MSE &
  0.15181281 &
  \textbf{0.948143$\pm$0.004261} &
  \textbf{3005} \\ 
 & 
  Fully-Connected &
  4 EdgeConv (embed dim: 16) &
  Custom $p_T$ &
  0.12730886 &
  \textbf{0.852538$\pm$0.006649} &
  \textbf{3005} \\ \midrule
\multirow{2}{*}{GNN (each feature as a node)} &
  $0 \rightarrow 1 \rightarrow 2 \rightarrow 3$ || $2 \leftrightarrow \{0, 4, 5, 6\}$, 6 &
  4 GCN (embed dim: 128) &
  $p_T$ &
  4.38785421 &
  45.339565$\pm$1.034975 &
  55076 \\
 &
  Fully-Connected &
  4 MPL (embed dim: 128) &
  $p_T$ &
  3.79453528 &
  41.708584$\pm$0.296089 &
  99952 \\ 
 &
  Fully-Connected &
  4 MPL (embed dim: 128) &
  $p_T$ &
  4.09143839 &
  44.749077$\pm$0.431576 &
  3005 \\ \midrule
\multirow{5}{*}{GNN (bending angle as a node)} &
  Fully-Connected &
  4 MPL (embed dim: 24) &
  MSE &
  4.005916 &
  1.274056$\pm$0.002062 &
  5579 \\
 &
  Fully-Connected &
  4 MPL (embed dim: 24) &
  MSE &
  0.25188804 &
  1.264197$\pm$0.016950 &
  5903 \\ 
 &
  Fully-Connected &
  4 MPL (embed dim: 24) &
  MSE &
  0.24324974 &
  1.242137$\pm$0.010045 &
  6437 \\ 
 &
  Fully-Connected &
  2 MPL (embed dim: 24) &
  MSE &
  0.24161860 &
  1.235898$\pm$0.017151 &
  6112 \\ 
 &
  Fully-Connected &
  4 MPL (embed dim: 28) &
  MSE &
  0.24752072 &
  1.250935$\pm$0.000345 &
  6545 \\ 
 &
  Fully-Connected &
  4 EdgeConv (embed dim: 16) &
  MSE &
  0.49968712 &
  1.785861$\pm$0.042125 &
  2909 \\ 
 &
  Fully-Connected &
  4 EdgeConv (embed dim: 16) &
  Custom $p_T$ &
  0.49823061 &
  1.779759$\pm$0.045607 &
  2909 \\ \midrule
\multirow{5}{*}{GNN ($\eta$ value as a node)} &
  Fully-Connected &
  4 MPL (embed dim: 24) &
  MSE &
  0.20895882 &
  1.146910$\pm$0.024018 &
  5579 \\
 &
  Fully-Connected &
  4 MPL (embed dim: 24) &
  MSE &
  0.16312774 &
  0.992087$\pm$0.002341 &
  5903 \\
 &
  Fully-Connected &
  4 MPL (embed dim: 24) &
  MSE &
  0.21127280 &
  1.145697$\pm$0.000922 &
  6437 \\
 &
  Fully-Connected &
  2 MPL (embed dim: 24) &
  MSE &
  0.20675154 &
  1.133285$\pm$0.019203 &
  6112 \\
 &
  Fully-Connected &
  4 MPL (embed dim: 28) &
  MSE &
  0.14937276 &
  \textbf{0.941620$\pm$0.023914} &
  \textbf{6545} \\ 
 &
  Fully-Connected &
  4 EdgeConv (embed dim: 16) &
  MSE &
  0.99398496 &
  2.987811$\pm$0.012732 &
  2909 \\
 &
  Fully-Connected &
  4 EdgeConv (embed dim: 16) &
  Custom $p_T$ &
  0.99399586 &
  2.989572$\pm$0.042993 &
  2909 \\ \bottomrule[1.5pt]
\end{tabular}%
}
\end{table}

\begin{figure}[!ht] 
\centering
  \subfloat[\label{1a} $p_T$ residuals distribution]{%
       \includegraphics[width=0.5\linewidth]{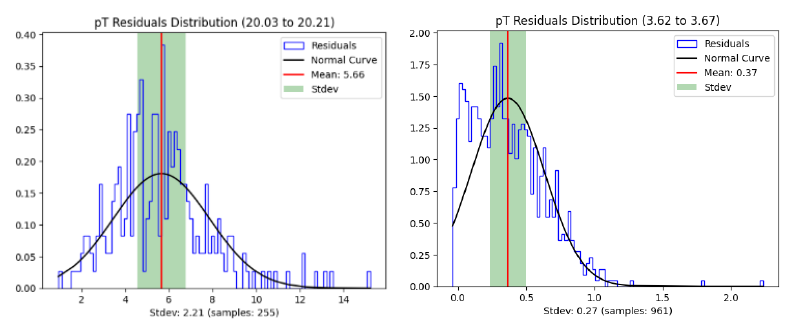}}
    \hfill
  \subfloat[\label{1c} Frequency heatmap of ground truth and predicted $p_T$]{%
        \includegraphics[width=0.47\linewidth]{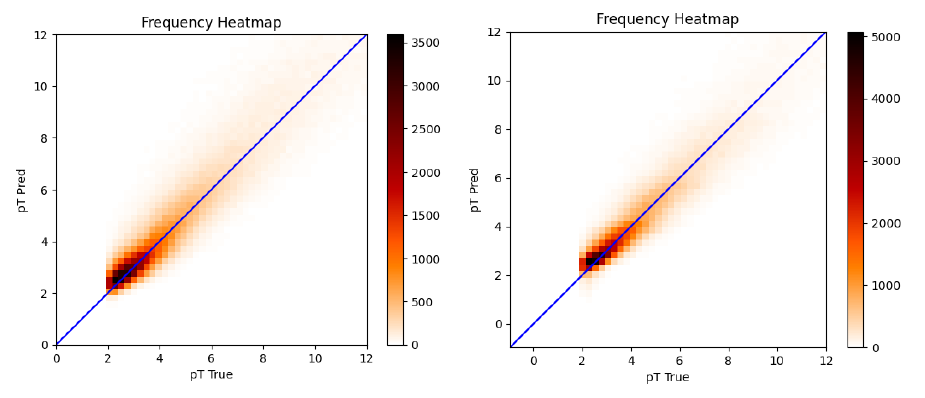}}
  \caption{(a) Gaussian distribution and (b) Frequency heatmap of TabNet (\textcolor{blue}{left}) and best-performing 4 EdgeConv-based (embedding dimension: 16) custom loss function optimized model (right). The EdgeConv model (\textcolor{myred}{right}) demonstrates significantly more concentrated residuals (lower standard deviation, mean closer to zero) and a tighter alignment between predicted and true $p_T$ values in the heatmap, indicating superior predictive accuracy.}
  \label{fig1} 
\end{figure}

\section{Experimental Setup}
\label{experimental_setup}
All experiments were conducted in a controlled environment with access to NVIDIA Tesla P100 GPU acceleration, 16GB VRAM, and standard PyTorch~\cite{pytorch} libraries for DL models and GNNs~\cite{torch_geometric,DGL,torch_geometric_temporal}. The training pipeline was designed to automatically utilize available CUDA-enabled devices for efficient computation. Model checkpoints, logs, and outputs were saved locally during training to ensure reproducibility and consistent evaluation across runs. The software environment was standardized using Python 3.7.12. GNN implementations and operations relied on PyTorch version 1.12.1~\cite{pytorch} in conjunction with the PyTorch Geometric library~\cite{torch_geometric}, version 2.1.0. Data preprocessing tasks, including feature standardization, were conducted using scikit-learn version 1.0.2~\cite{scikit-learn}. We applied a consistent training pipeline across all models, using the Adam optimizer with a learning rate of $1\times10^{-4}$ and weight decay of $5\times10^{-4}$, trained for up to 50 epochs. A `ReduceLROnPlateau'\footnote{\url{https://docs.pytorch.org/docs/stable/generated/torch.optim.lr_scheduler.ReduceLROnPlateau.html}} scheduler dynamically adjusted the learning rate based on validation performance, and early stopping was employed to prevent overfitting. The best-performing weights were saved based on the minimum validation loss.

Datasets were partitioned into training and testing sets. For Methods 1 and 2, which utilized the full dataset of 1,179,356 samples, an 80:20 split resulted in 943,484 training samples and 235,872 testing samples. For Methods 3 and 4, which operated on the outlier-filtered dataset of 1,029,592 samples, a similar 80:20 split yielded 823,673 training samples and 205,919 testing samples (the slight discrepancy from 823,673 is due to rounding in the split). Reproducibility across all experiments was ensured by initializing random number generators with a fixed seed for data splitting and model weight initialization, where applicable.

\subsection{Baseline Model Architectures}
\label{append:baseline}
To benchmark the proposed physics-informed GNNs, we evaluated several standard DL models. These included FCNNs, CNNs, LSTM networks, and TabNet~\cite{tabnet_2021}. All models were trained using MSE as the loss function. We used MAE for evaluation. Each model employed a linear activation function in its output layer for regression.

\subsubsection{TabNet}
TabNet~\cite{tabnet_2021}, a strong baseline for tabular data, was included. Its architecture uses feature-wise attention. It selects features dynamically over multiple decision steps. Each step includes a feature transformer, an attentive transformer, and a feature masking module. The model processed flattened input features. Final predictions were aggregated from all decision steps.

\subsubsection{Fully Connected Neural Network (FCNN)}
The FCNN processes flattened input features through a series of dense layers. It accepts an input tensor of a defined shape (e.g., 31 features if using a combined feature set). The architecture consists of five sequential hidden dense layers: the first with 512 units, the second with 256 units, and the next three layers each with 128 units. All hidden dense layers utilize ReLU activation. After the hidden layers, a final dense output layer with a specified number of units (e.g., a single unit for regression tasks) employs a linear activation function to produce the prediction.

\subsubsection{Convolutional Neural Network (CNN)}
The CNN received input tensors of shape 7 $\times$ 4 $\times$ 1. It began with two 2D Convolutional (Conv2D) layers with 64 and 128 filters, each using 2 $\times$ 2 kernels. Leaky Rectified Linear Unit (LeakyReLU)~\cite{leakyrelu} activation ($\alpha$ = 0.15) followed each Conv2D layer. A MaxPooling2D layer with a pool size of 2 $\times$ 1 was then applied. Four additional Conv2D layers with 256, 256, 128, and 128 filters followed. After flattening, two dense layers with 256 and 128 units were used, both with LeakyReLU activation and Dropout (rate = 0.5). A linear output layer produced the final result.

\subsubsection{CNN-Grid}
The Parameterized CNN-Grid model processes 2D input features (e.g., shape \(7 \times 4 \times 1\)) through a sequence of convolutional and dense layers, with all hidden layers utilizing LeakyReLU activation with an alpha of 0.15. The architecture begins with two initial Conv2D layers: the first with 64 filters and the second with 128 filters, both using \(2\times2\) kernels and `same' padding. This is followed by a MaxPooling2D layer with a pool size of (2,1). Subsequently, a configurable number of additional Conv2D layers (e.g., one additional layer if \( \text{conv} - 1 = 1\), resulting in a total of three Conv2D layers in this block), each with 128 filters, \(2 \times 2\) kernels, and `same' padding, are applied, followed by another MaxPooling2D layer with a pool size of (2,2). After flattening the feature map, the data passes through an initial dense layer with 128 units, followed by a Dropout layer with a rate of 0.5. Then, a configurable number of additional dense layers (e.g., two additional layers if \( \text{FCN} = 2\), for a total of three dense layers in this block), each with 128 units, are applied, with each dense layer followed by a Dropout layer (rate = 0.5). Finally, a dense output layer with a specified number of units (e.g., a single unit for regression) employs a linear activation function to produce the prediction.

\subsubsection{Dual Channel Network (CNN-FCNN)}
The architecture, termed ``Multi,'' processes distinct inputs through two parallel branches using ReLU activation in all hidden layers. The first branch, a CNN, receives input tensors of shape (e.g., \(7 \times 4 \times 1\)) and processes them through three sequential Conv2D layers: the first with 32 filters, the second with 64 filters, and the third with 64 filters, all using \(2 \times 2\) kernels and same padding with ReLU activation. MaxPooling is applied after the second Conv2D layer, followed by a flattening operation. Concurrently, the second branch, an FCNN, processes auxiliary features (e.g., three features) through four sequential dense (FCN) layers with 128, 128, 64, and 64 units, respectively, all using ReLU activation. The outputs from both branches are concatenated and passed through a fusion dense layer with 32 units and ReLU activation. Finally, a single-unit linear dense output layer produces the regression prediction.

\subsubsection{Long Short-Term Memory (LSTM) Network}
The LSTM model treats the input as a sequence with 7 time steps, each containing 4 features. It includes two LSTM layers: the first with 64 units (with \texttt{return\_sequences} enabled), followed by the second with 32 units. A Dropout layer with a rate of 0.2 follows both LSTM layers. A dense layer with 64 units and ReLU activation precedes the final single-unit linear output layer.

\section{Results and Discussion}
\label{result_discussion}
Our exploration of four distinct physics-informed graph construction strategies for GNNs (Table \ref{tab:my-table}) establishes them as highly effective alternatives to DL models, particularly \textit{TabNet} (MAE: 0.9607, $\approx$6.7k parameters), for $p_T$ estimation. These strategies consistently demonstrate improved accuracy with significant parameter efficiency. We chose CNNs, LSTMs, FCNNs, and hybrid architectures as DL baselines due to their proven effectiveness in regression tasks across various domains and their ongoing exploration by CMS for trigger-level applications requiring microsecond-latency and optimized inference~\cite{Feng_2025}. Among six DL baselines, \textit{TabNet} outperformed others by a large margin, so we chose it for further comparisons. Notably, the \textit{station-as-node} graph construction, when paired with an EdgeConv backbone and custom loss, yielded a state-of-the-art MAE of 0.8525 using only $\approx$3k parameters, which is a $\ge55\%$ parameter reduction and superior accuracy compared to all DL baselines, particularly \textit{TabNet}, with its improved prediction precision visually corroborated by Figure \ref{fig1}. Treating each station as a node significantly outperforms the \textit{feature-as-node} variant, likely due to richer 7-dimensional node features, enabling more effective representation learning than the 4-dimensional features in the latter. Moreover, the $\eta$-centric graph strategy, using our novel MPL architecture, also outperforms TabNet, achieving an MAE of 0.9416 with $\approx$6.5k parameters. These findings show that physics-informed GNNs are well-suited for capturing spatial and angular patterns in muon detector data and can deliver accurate, efficient $p_T$ estimation under hardware constraints.

\section{Conclusions}\label{conclusion}
We introduced physics-informed GNNs for real-time $p_T$ estimation in CMS trigger systems. By systematically integrating detector geometry and physics priors into four distinct graph construction strategies, developing a novel MPL, and employing custom domain-specific loss functions, our models significantly outperform established baselines. Notably, a station-informed EdgeConv model achieved a state-of-the-art MAE of 0.8525 with $\approx3k$ parameters, representing a $\ge55\%$ parameter reduction and superior accuracy compared to DL baselines, particularly TabNet. This work underscores the effectiveness of co-designing GNN architectures with domain-specific knowledge, paving the way for robust, accurate, and deployable AI solutions in the demanding, resource-constrained environments of HEP trigger systems.

\paragraph{Limitations}\label{limitations}
This study primarily utilizes simulated data from the CMS Trigger Dataset; performance validation on real detector data, which may introduce additional unmodeled noise and systematic effects, is a necessary next step. While resource-constrained deployment on FPGAs is a key motivation, detailed implementation and on-hardware latency/resource utilization analysis for all proposed GNN architectures (especially the more complex MPL variants) are beyond the current scope and represent an important avenue for future work. The generalizability of the specific graph construction strategies and custom loss functions to other HEP experiments or different particle physics tasks would also benefit from further investigation. Finally, while significant parameter reductions were achieved for the best-performing model, the trade-off between model complexity, achievable latency, and the stringent constraints of real-time trigger hardware warrants continued exploration for broader model deployment.

\paragraph{Broader Impacts}\label{broader_impact}
The primary broader impact of this research lies in advancing real-time data processing capabilities for HEP experiments, such as the CMS at CERN. By enabling more accurate and efficient $p_T$ estimation, particularly in high-pileup environments and on resource-constrained hardware like FPGAs, our work can significantly improve the quality and quantity of data selected by trigger systems. This, in turn, can enhance the potential for new physics discoveries by allowing for more precise event selection at the earliest stages of data acquisition. The development of physics-informed GNNs and efficient model architectures, such as the novel MPL, also offers transferable insights for other scientific domains requiring fast and accurate analysis of complex, structured data under strict computational budgets. While the direct societal impact is research-focused, the methodologies contribute to the broader field of efficient AI. No direct negative societal impacts are foreseen from this specific application, as its primary goal is to improve fundamental scientific research.



\section*{Data availability}
The datasets used in this study are publicly available on Kaggle at the following link: \\
\href{https://www.kaggle.com/datasets/ekurtoglu/cms-dataset}{https://www.kaggle.com/datasets/ekurtoglu/cms-dataset}. 


\section*{Funding sources}
This research did not receive any specific grant from funding agencies in the public, commercial, or not-for-profit sectors.

\section*{Author contributions statement}
\textbf{M.A.J.}: Conceptualization, Methodology, Data curation, Writing - Original Draft Preparation, Software, Visualization, Investigation, Validation.
\textbf{S.S.}: Writing - Original Draft Preparation, Software, Investigation, Validation.
\textbf{M.F.M.}: Supervision, Writing - Review \& Editing.
\textbf{M.M.M. \& M.A.H.}: Supervision, Writing - Review \& Editing.

\section*{Competing interests}
The authors have no conflict of interest, financial or otherwise, to declare relevant to this article.

\section*{Additional information}
Correspondence and requests for materials should be addressed to M.A.J.

\bibliographystyle{elsarticle-num} 
\bibliography{main}

\end{document}